\documentclass[11pt]{article}

\usepackage[final]{acl}

\usepackage{times}
\usepackage{latexsym}
\usepackage{tabularx}
\usepackage{array}
\usepackage{graphicx}
\usepackage{amssymb}
\usepackage{amsmath}
\usepackage{caption}
\usepackage{standalone}
\usepackage{comment}
\newcommand{\model}{\textsc{SWYB}\xspace}
\newcommand{\mthdname}{\textsc{SWYB}\xspace}
\usepackage{placeins}

\usepackage{amsmath}
\usepackage{algpseudocode}
\usepackage{bm} 
\usepackage{amsthm}

\usepackage[ruled,vlined,linesnumbered]{algorithm2e}
\DontPrintSemicolon

\usepackage[T1]{fontenc}

\usepackage[utf8]{inputenc}

\usepackage{microtype}
\usepackage{booktabs}
\usepackage{tikz}
\usepackage{listofitems}
\usepackage{etoolbox}

\usepackage{adjustbox}
\usepackage{pgfplots, pgfplotstable}
\pgfplotsset{compat=1.18}
\usetikzlibrary{arrows.meta, positioning, automata, calc}

\definecolor{layer1}{RGB}{173,216,230}
\definecolor{layer2}{RGB}{144,238,144}
\definecolor{layer3}{RGB}{255,182,193}
\colorlet{layer1border}{layer1!70!black}
\colorlet{layer2border}{layer2!70!black}
\colorlet{layer3border}{layer3!70!black}
\definecolor{barA}{RGB}{173,216,230}
\definecolor{barB}{RGB}{144,238,144}
\definecolor{barC}{RGB}{255,182,193}
\definecolor{barD}{HTML}{ff362e}
\definecolor{barE}{HTML}{7982db}

\tikzset{
  neuron/.style={circle, thick, minimum size=0.7cm, inner sep=1pt},
  connect/.style={->, >=stealth, thick, draw=black!70},
}

\tikzset{database/.style={cylinder,aspect=0.5,draw,rotate=90,path picture={
\draw (path picture bounding box.160) to[out=180,in=180] (path picture bounding
box.20);
\draw (path picture bounding box.200) to[out=180,in=180] (path picture bounding
box.340);
}}}

\usepackage{inconsolata}

\usepackage{graphicx}

\newtheorem{theorem}{Theorem}[section]

\title{Stay Within Your Bounds: Distance-Guided Decoding for Guaranteed Context-Free Grammar Compliance}

\author{\textbf{Vincenzo Collura}\textsuperscript{1,*}, \textbf{Karim Tit}\textsuperscript{1,*}, \textbf{Eleonora Giunchiglia}\textsuperscript{2}, \\ \textbf{Mike Papadakis}\textsuperscript{1}, \textbf{Maxime Cordy}\textsuperscript{1} \\ \textsuperscript{1} University of Luxembourg, Luxembourg \\ \textsuperscript{2} Imperial College London, United Kingdom \\ \textsuperscript{*} Equal contribution. \\ \small \textbf{Correspondence:} \href{mailto:vincenzo.collura@uni.lu}{vincenzo.collura@uni.lu}
}

\begin{document}
\maketitle
\begingroup
\renewcommand{\thefootnote}{}
\footnotemark\footnotetext{The code to run all our experiments is available on GitHub: 
\href{https://github.com/Enzossssss/SWYB}{https://github.com/Enzossssss/SWYB}.}
\endgroup

\begin{abstract}
Grammar-constrained decoding helps large language models produce syntactically valid structured outputs, such as code, JSON, and SQL. For context-free grammars, many practical decoders enforce local prefix feasibility: each token must keep the current prefix extendable to some valid completion. Yet, under tokenizer–grammar mismatch and finite token budgets, feasible prefixes may still fail to reach acceptance. We propose a lookahead-guided decoding framework for context-free grammars based on pushdown automata. Offline, we compute bounded pushdown summaries with reachability labels and upper-bound distances to acceptance. Online, these estimates guide horizon-aware pruning and beam search. The resulting decoder is syntactically sound: every output is accepted by the target grammar. Experiments on JSON, SQL, and Linear Temporal Logic (LTL) show both consistent syntactic validity and improved completion quality over existing baselines.
\end{abstract}
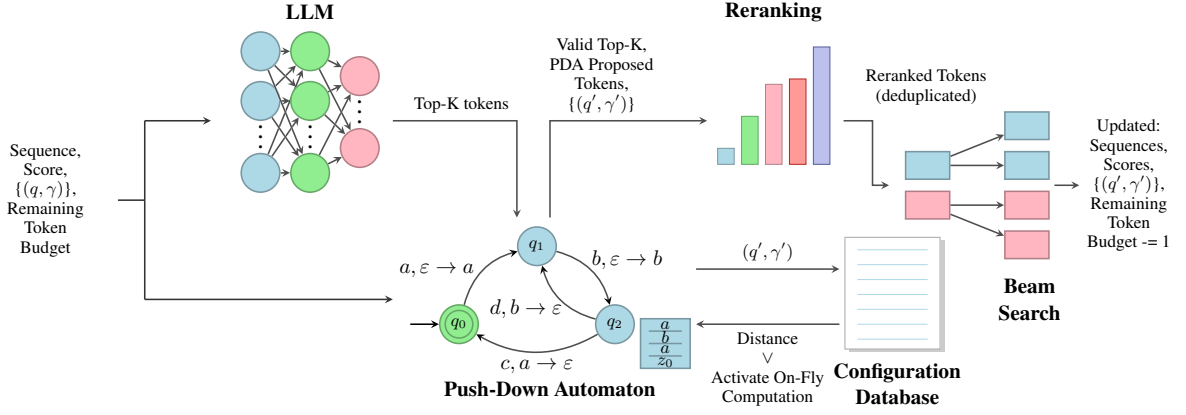
\begin{figure*}
    \centering
    \adjustbox{max width=\textwidth,keepaspectratio}{%
    \begin{tikzpicture}[x=1.5cm,y=0.9cm]
      \readlist\Nnod{3,3,2}
      \def\yshift{0.45}
      
      \begin{scope}[x=0.6*1.5cm, y=0.9cm]
        \foreachitem \N \in \Nnod{
          \def\lay{\Ncnt}
          \pgfmathsetmacro\prev{int(\Ncnt-1)}
          \ifnum\lay=1\relax
            \def\layercolor{layer1}
            \def\bordercolor{layer1border}
          \else\ifnum\lay=2\relax
            \def\layercolor{layer2}
            \def\bordercolor{layer2border}
          \else
            \def\layercolor{layer3}
            \def\bordercolor{layer3border}
          \fi\fi
          \foreach \i [evaluate={\c=int(\i==\N); \y=\N/2-\i-\c*\yshift; \x=\lay;}] in {1,...,\N}{
            \node[neuron, fill=\layercolor, draw=\bordercolor] (N\lay-\i) at (\x,\y) {};
            \ifnumcomp{\lay}{>}{1}{
              \foreach \j in {1,...,\Nnod[\prev]}{
                \draw[white,line width=1,shorten >=1] (N\prev-\j) -- (N\lay-\i);
                \draw[connect] (N\prev-\j) -- (N\lay-\i);
              }
            }{}
          }
          \path (N\lay-\N) --++ (0,0.9+\yshift) node[midway,scale=1.3] {$\vdots$};
        }
      \end{scope}
      
      \begin{scope}[yshift=-4.5cm, xshift=4.5cm,
        every state/.style={draw=red!80!black, thick, minimum size=0.7cm, inner sep=2pt, circle, font=\small},
        every initial by arrow/.style={font=\small, thick, -stealth},
        initial,
        auto,
        node distance=2.0cm,
        on grid,
      ]

        \node (q0) [state, initial, initial text=, fill=layer2, draw=layer2border] 
          [minimum size=0.7cm] {$q_0$};
        \begin{scope}[every node/.style={circle, draw=layer2border, thick, fill=none}]
          \node at (q0) [minimum size=0.5cm, draw=layer2border, fill=none] {};
        \end{scope}
        
        \node (q1) [state, minimum size=0.7cm, fill=layer1, draw=layer1border, above right = of q0] {$q_1$};
        \node (q2) [state, minimum size=0.7cm, fill=layer1, draw=layer1border, below right = of q1] {$q_2$};
      
        \path [-stealth, thick, text=black, connect]
          (q0) edge [bend left] node [left=0.02cm] {$a,\varepsilon \rightarrow a$} (q1)
          (q1) edge [bend left] node[above right, inner sep=0pt] {$b,\varepsilon \rightarrow b$} (q2)
          (q2) edge [bend left] node[below] {$c,a \rightarrow \varepsilon$}  (q0)
          (q2) edge [bend left] node[below left, inner sep=-1pt] {$d,b \rightarrow \varepsilon$} (q1);
        
        \begin{scope}[shift={(q2.east)}, xshift=0.1cm, yshift=-0.8cm]
          \draw[thick, fill=layer1, draw=layer1border] (0,0) rectangle (0.6,1.0);
          \draw[thick, draw=layer1border] (0.1,0.25) -- (0.5,0.25);
          \draw[thick, draw=layer1border] (0.1,0.5) -- (0.5,0.5);
          \draw[thick, draw=layer1border] (0.1,0.75) -- (0.5,0.75);
          \node at (0.3, 0.875) {\small$a$};
          \node at (0.3, 0.625) {\small$b$};
          \node at (0.3, 0.375) {\small$a$};
          \node at (0.3, 0.125) {\small$z_0$};
        \end{scope}
        
      \end{scope}
      
      \begin{scope}[every node/.style={align=center, font=\small, draw=white},
                    every path/.style={->, >=stealth, thick, draw=black!80}]
        \node (input) at (-2,-2.5) [draw, rounded corners, fill=white, 
              minimum width=2.6cm, minimum height=1.8cm, inner sep=2pt] 
              {Sequence,\\Score,\\\{$(q, \gamma)$\},\\Remaining\\Token\\Budget};
        
        \coordinate (branch) at (-0.8,-2.5);
        \coordinate (premiddlenn) at (-0.8,-0.9);
        \coordinate (middlenn) at (0,-0.9);
        \coordinate (premiddleautomaton) at (-0.8,-4.5);
        \coordinate (middleautomaton) at (2.2,-4.5);
        \draw[-] (input.east) -- (branch);
        \draw[-] (branch) -- (premiddlenn);
        \draw[->] (premiddlenn) -- (middlenn);
        \draw[-] (branch) -- (premiddleautomaton);
        \draw[->] (premiddleautomaton) -- (middleautomaton);
      \end{scope}
      
       \begin{scope}[every node/.style={font=\small, text=black}]
        \draw[connect] (2.2, -0.9) -- (3.7, -0.9) -- (3.7, -2.85)
          node[pos=-0.0, above left, align=center] {Top-K tokens};
      \end{scope}
      
        \begin{scope}[xshift=8.5cm, yshift=-1.6cm]
            \begin{axis}[
                ybar=-0.4cm,
                bar width=0.3cm,
                width=5cm,
                height=4cm,
                symbolic x coords={a, b, c, d, e},
                xtick=\empty,
                ytick=\empty,
                xlabel={},
                ylabel={},
                ymin=0,
                ymax=2.5,
                tick style={draw=none},
                enlarge x limits=0.3,
                axis lines=none,
            ]
                \addplot[barA,fill, draw=barA!70!black] coordinates {(a,0.3)};
                \addplot[barB,fill, draw=barB!70!black] coordinates {(b,0.9)};
                \addplot[barC,fill, draw=barC!70!black] coordinates {(c,1.5)};
                \addplot[fill=barD!50!white, draw=barD!70!black] coordinates {(d,1.6)};
                \addplot[fill=barE!50!white, draw=barE!70!black] coordinates {(e,2.2)};
            \end{axis}
        \end{scope}
      
      \begin{scope}[shift={(13cm, -1.8cm)}]
        \tikzset{beamnode/.style={rectangle, thick, minimum width=0.8cm, minimum height=0.5cm, align=center, inner sep=2pt}}
        
        \node[beamnode, fill=layer1, draw=layer1border] (init1) at (0,0.2) {};
        \node[beamnode, fill=layer1, draw=layer1border] (c11) at (1.2,1.0) {};
        \node[beamnode, fill=layer1, draw=layer1border] (c12) at (1.2,0.2) {};
        
        \node[beamnode, fill=layer3, draw=layer3border] (init2) at (0,-0.6) {};
        \node[beamnode, fill=layer3, draw=layer3border] (c21) at (1.2,-0.6) {};
        \node[beamnode, fill=layer3, draw=layer3border] (c22) at (1.2,-1.4) {};
        
        \draw[->, thick][connect] (init1) -- (c11);
        \draw[->, thick][connect] (init1) -- (c12);
        \draw[->, thick][connect] (init2) -- (c21);
        \draw[->, thick][connect] (init2) -- (c22);
      \end{scope}
      
    \draw[->, thick, draw=black!70][connect] 
    (7.65,-0.9) -- (8.0,-0.9) 
    node[above right, midway, align=center, font=\small, yshift=0.15cm] {Reranked Tokens\\(deduplicated)} 
    -- (8.0,-2.2) -- (8.2,-2.2);
          
      \begin{scope}[every node/.style={font=\small, text=black}]
        \draw[connect] (4.1, -2.85) -- (4.1, -0.9) -- (6.0, -0.9) 
          node[pos=-0.05, above right, align=center] {Valid Top-K,\\PDA Proposed\\Tokens,\\\{$(q', \gamma')$\}};
      \end{scope}

      \draw[->, thick, draw=black!70][connect] (10.2, -2.2) -- (10.5, -2.2) node[right, align=center, font=\small, name=beamout] {Updated:\\Sequences,\\Scores,\\\{$(q', \gamma')$\},\\Remaining\\Token\\Budget -= 1};
      
      \begin{scope}[shift={(11.5cm, -4.95cm)}, scale=0.5]

          \def\W{2.3}
          \def\H{4.5}
        
          \def\leftMargin{0.3}
          \def\rightMargin{0.3}
          \def\startY{0.4}
          \def\bottomMargin{0.4}
          \def\lineSpacing{0.6}
        
          \pgfmathsetmacro{\endY}{\H - \bottomMargin}
        
          \fill[gray!40] (0.1,-0.1) rectangle (\W+0.1, \H-0.1);
          \fill[white] (0,0) rectangle (\W,\H);
          \draw[gray!60, line width=0.6pt] (0,0) rectangle (\W,\H);
        
          \pgfmathsetmacro{\numLines}{(\endY - \startY)/\lineSpacing}
          \pgfmathtruncatemacro{\numLinesInt}{\numLines}
        
          \foreach \i in {0,...,\numLinesInt} {
            \pgfmathsetmacro{\y}{\startY + \i * \lineSpacing}
            \draw[layer1, line width=0.5pt] (\leftMargin, \y) -- (\W-\rightMargin, \y);
          }
    \end{scope}

    \draw[connect] 
        (8.8cm, -3.5cm) -- (11.4cm, -3.5cm)
        node[midway, above, font=\small] {$(q', \gamma')$};
    
    \draw[connect] 
        (11.4cm, -4.5cm) -- (8.8cm, -4.5cm)
        node[
            midway,
            below,
            font=\small,
            align=center
        ]
        {Distance\\
         $\lor$\\
         Activate On-Fly\\Computation};
      
    \node[font=\normalsize\bfseries, text=black] at (1.2, 1.3) {LLM};
    \node[font=\normalsize\bfseries, text=black] at (4.1, -6.3) {Push-Down Automaton};
    \node[font=\normalsize\bfseries, text=black, align=center] at (8.3, -6.2) {Configuration\\Database};
    \node[font=\normalsize\bfseries, text=black] at (6.8, 1.3) {Reranking};
    \node[font=\normalsize\bfseries, text=black, align=center] at (9.9, -4.5) {Beam\\Search};
      
    \end{tikzpicture}
    }
    \caption{
Overview of \model: Model top-$K$ tokens are validated by tokenizer-aware PDA execution, filtered using reachability and distance-to-acceptance estimates, and combined with PDA-proposed candidates. Surviving tokens are reranked and used to update the beam state.
}
    \label{fig:neurons_colored}
\end{figure*}
\section{Introduction}
Large language models (LLMs) are increasingly used to generate structured objects such as programs, queries, and data records. In these settings, syntactic well-formedness is not merely a formatting preference: invalid outputs may be rejected by downstream parsers, violate interface contracts, or fail to
execute. This has motivated a growing line of work on constrained and guided decoding, including grammar-constrained generation, type-constrained code generation, and structured generation engines for LLMs~\citep{
willard2023efficientguidedgenerationlarge,
beurerkellner2024guidingllmsrightway,
ugare2024syncodellmgenerationgrammar,
dong2025xgrammarflexibleefficientstructured,
park2025flexibleefficientgrammarconstraineddecoding}.

For context-free grammar constraints, many practical decoders enforce validity through \emph{local prefix feasibility}. At each step, they compute which tokens keep the current prefix extendable to at least one valid completion and mask the
remaining tokens. This mechanism is effective for preventing immediate syntax violations and has been implemented in several recent systems for structured generation~\citep{
willard2023efficientguidedgenerationlarge,
beurerkellner2024guidingllmsrightway,
ugare2024syncodellmgenerationgrammar,
dong2025xgrammarflexibleefficientstructured,
park2025flexibleefficientgrammarconstraineddecoding,
Chen_2025}. However, prefix feasibility is weaker than completed-output validity under a finite token budget. A prefix may remain locally admissible while moving toward configurations that require too many remaining tokens, additional structural commitments, or low-probability continuations. Thus, local masking preserves prefix validity, but it does not quantify progress toward acceptance.

A complementary line of work studies global control of autoregressive generation using automata, dynamic programming, probabilistic steering, or sampling-based inference~\citep{
zhang2023tractablecontrolautoregressivelanguage,
zhang2024adaptablelogicalcontrollarge,
semctrl,
loula2025syntacticsemanticcontrollarge,
lipkin2025fastcontrolledgenerationlanguage}. These methods show that constrained generation can benefit from reasoning about future completions. For context-free grammar constraints, however, such lookahead is difficult: the natural
operational model is a pushdown automaton (PDA), whose configurations include an unbounded stack. This yields an infinite configuration space, making direct finite-state value computation lookahead inapplicable.

In this paper, we introduce \mthdname{}, a lookahead-guided decoding framework for LLM generation under context-free grammar constraints. \mthdname{} augments local grammar masking with distance-to-acceptance estimates computed over a pushdown representation. Offline, we construct bounded stack summaries using ideas from the reachability analysis of pushdown systems~\citep{pushdown_systems} and weighted pushdown systems~\citep{weighted_pushdown_systems}. These summaries associate abstract PDA configurations with reachability labels
and token-level upper-bound distances to acceptance. Online, the decoder uses these estimates for pruning and re-ranking in the beam search.

A practical difficulty is that LLM vocabularies are defined over subword tokens, whereas grammars are usually specified over characters, terminals, or lexical units. \mthdname{} therefore includes a tokenizer-aware consumption procedure for matching model tokens against PDA transitions. At each decoding step, it
combines model top-$k$ tokens with automaton-proposed backup tokens, validates the resulting candidates through budgeted successor exploration, and uses distance information to favor continuations that make structural progress. The resulting decoder is syntactically sound: every completed output returned by the algorithm is accepted by the target PDA.

In summary, our contributions are:
\begin{itemize}
    \item A sound pushdown lookahead mechanism based on bounded reachability and token-distances to acceptance.
    \item A tokenizer-aware decoding algorithm combining model top-$k$, automaton proposals, and budgeted PDA exploration.
    \item Experiments on JSON, SQL and LTL grammars showing consistent syntactic correctness and improved completion quality.
\end{itemize}

\section{Related Work}
\label{sec:related-work}

Constrained decoding methods differ along three axes: the constraint class they support, whether they enforce prefix feasibility, and whether they use lookahead or remaining-budget information. Table~\ref{tab:related-work} summarizes this landscape. \mthdname{} targets a missing combination in
practical systems: CFG/PDA constraints with prefix safety and budget-aware acceptance guidance.

\paragraph{Grammar masking.}
Grammar-constrained decoding restricts the output tokens of the language model so that the completion belongs to a given formal language. Practical systems such as Guidance and Outlines-style guided generation support structured generation through regular expressions, templates, and grammar-like specifications~\citep{lundberg2024guidance,willard2023efficientguidedgenerationlarge}. Recent work focuses on efficient CFG enforcement under subword tokenization: DOMINO addresses tokenizer-grammar alignment~\citep{beurerkellner2024guidingllmsrightway}, SynCode builds grammar-aware lookup structures~\citep{ugare2024syncodellmgenerationgrammar}, XGrammar optimizes grammar execution and token-mask construction~\citep{dong2025xgrammarflexibleefficientstructured}, GreatGramma studies tokenizer-grammar preprocessing~\citep{park2025flexibleefficientgrammarconstraineddecoding},
and Pre$^3$ exploits deterministic pushdown automata for LR(1)
grammars, which constitute a strict subset of CFGs~\citep{Chen_2025}.

These methods primarily enforce \emph{local prefix feasibility}: at each step, the decoder checks whether a token keeps the current prefix extendable to at least one valid completion. This provides prefix safety, but not necessarily completed-output validity under a finite token budget. If decoding stops because the maximum number of tokens is reached before an accepting configuration, the returned output may still be syntactically invalid, even though every generated prefix was locally feasible~\citep{ugare2024syncodellmgenerationgrammar}. Thus, local masking answers whether a prefix can still be completed, but not whether it is close enough to acceptance within the remaining budget.

\paragraph{Global control.}
A complementary line of work studies controlled generation through global objectives, automata-based dynamic programming, or probabilistic inference. GeLaTo formulates tractable control for autoregressive generation with automata and dynamic programming~\citep{zhang2023tractablecontrolautoregressivelanguage}. Ctrl-G extends this view to adaptable logical control~\citep{zhang2024adaptablelogicalcontrollarge}. SEM-CTRL steers generation toward semantic constraints~\citep{semctrl}. ABS enforces constraints that can be represented as a deterministic finite automaton (DFA) using the distance from the accepting states~\citep{collura2025absenforcingconstraintsatisfaction}.  Following a similar approach, TruncProof addresses finite-budget constrained decoding using LL(1) parsing and remaining-token-cost estimates~\citep{kato2026truncproof}.
Importantly, LL(1) grammars form a \textit{strict} subset of deterministic CFGs (DCFGs), which are themselves
a strict subset of general CFGs. In contrast, our method, \mthdname{}, operates on generic
PDAs and therefore supports general CFG constraints.
 Sequential Monte Carlo and adaptive weighted rejection sampling provide sampling-based control mechanisms for syntactic, semantic, or black-box constraints~\citep{loula2025syntacticsemanticcontrollarge,lipkin2025fastcontrolledgenerationlanguage}. These methods show the value of lookahead, but global control is much harder for CFGs than for regular constraints because pushdown automata have unbounded stacks. 

\paragraph{Reachability analysis.}
The formal-methods literature provides symbolic tools for infinite-state
pushdown systems. Reachability for pushdown automata can be represented using
finite stack summaries and saturation procedures~\citep{pushdown_systems}.
Weighted pushdown systems extend this idea by attaching weights to transitions
and computing path quantities, with applications to interprocedural dataflow
analysis~\citep{weighted_pushdown_systems}. \mthdname{} adapts this perspective
to constrained decoding by computing bounded pushdown summaries and upper-bound
distances to acceptance, then using them to guide token pruning and re-ranking.
\paragraph{Program constraints.}
Several methods specialize in constrained decoding for code or programming-language structure. Type-constrained decoding restricts generation using type information~\citep{M_ndler_2025}, while correctness-guaranteed code generation studies decoding under stronger program-level specifications~\citep{li2025correctnessguaranteedcodegenerationconstrained}. CFG-constrained decoding has also been extended to diffusion language models~\citep{mündler2025constraineddecodingdiffusionllms}. These approaches are complementary to ours: they enrich the constraint language or model class, whereas \model adds \textit{sound} distance-guided lookahead to CFG/PDA decoding.
\begin{table}[t]
\centering
\small
\setlength{\tabcolsep}{2.5pt}
\renewcommand{\arraystretch}{1.05}
\begin{tabularx}{\columnwidth}{@{}>{\raggedright\arraybackslash}Xccc@{}}
\toprule
Method & Constraint & Prefix & Budget \\
\midrule
Guidance~\small\citep{lundberg2024guidance}
& CFG & partial & no \\
Outlines~\small\citep{willard2023efficientguidedgenerationlarge}
& CFG & yes & no \\
DOMINO~\small\citep{beurerkellner2024guidingllmsrightway}
& CFG & yes & no \\
SynCode~\small\citep{ugare2024syncodellmgenerationgrammar}
& CFG & yes & no \\
XGrammar~\small\citep{dong2025xgrammarflexibleefficientstructured}
& CFG & yes & no \\
GreatGramma~\small\citep{park2025flexibleefficientgrammarconstraineddecoding}
& CFG & yes & no \\
Pre$^3$~\small\citep{Chen_2025}
& DCFG & yes & no \\
GeLaTo~\small\citep{zhang2023tractablecontrolautoregressivelanguage}
& REG & yes & yes \\
Ctrl-G~\small\citep{zhang2024adaptablelogicalcontrollarge}
& REG & yes & yes \\
ABS~\small\citep{collura2025absenforcingconstraintsatisfaction}
& REG & yes & yes \\
GenLM.control~\small\citep{loula2025syntacticsemanticcontrollarge,lipkin2025fastcontrolledgenerationlanguage}
& CFG/BB & yes & yes \\
SEM-CTRL~\small\citep{semctrl}
& sem. & no & partial \\
Type-constr.~\small\citep{M_ndler_2025}
& program & yes & no \\
TruncProof~\small\citep{kato2026truncproof}
& LL(1) & yes & yes \\
\midrule
\model & CFG & yes & yes \\
\bottomrule
\end{tabularx}
\caption{
Positioning of {\mthdname{}} relative to constrained and controlled decoding methods. "Prefix" indicates whether generation is kept within the prefix closure of the constraint language. "Budget" indicates whether decoding uses a remaining-horizon, distance-to-acceptance, or related lookahead signal. 
}
\label{tab:related-work}
\end{table}

\paragraph{Budget saturation problem.} Local constrained decoding methods enforce constraints based primarily on the current decoding state, but lack an explicit notion of distance to an accepting configuration. Consequently, locally valid actions may lead to increasingly complex intermediate structures without providing guidance on when to stop expanding and start closing them. This myopia can cause performance to saturate as the generation budget increases: providing additional tokens does not necessarily bring the decoder closer to satisfying the global constraints, as can be observed empirically in Appendix~\ref{saturation}.

\section{CFG-Constrained Decoding as a PDA Reachability Problem}
\label{sec:setup}

We consider LLM decoding under a context-free grammar (CFG) constraint. Since
CFG and pushdown automata (PDA) define the same class of languages, any CFG
can be represented by an equivalent PDA, and conversely, any PDA recognizes a
context-free language~\citep{hopcroft2006automata}. We therefore formulate the
constraint through a PDA
\begin{equation}
\label{eq:pda}
\mathcal P = (Q,\Sigma,\Gamma,\Delta,q_0,z_0,F),
\end{equation}
where $Q$ is a finite set of control states, $\Sigma$ is the input alphabet
whose elements correspond to CFG terminals, $\Gamma$ is the stack alphabet,
$\Delta\subseteq Q \times (\Sigma \cup \{\varepsilon\}) \times \Gamma
\times Q \times \Gamma^*$ is the transition relation, $q_0\in Q$ is the initial state,
$z_0\in\Gamma$ is the initial stack symbol, and $F\subseteq Q$ is the set of
accepting states. A PDA configuration is a pair
$(q,\gamma)\in Q\times\Gamma^*$, where $\gamma$ is the stack, written with its
top on the left. We write
\begin{equation}
\label{eq:pda-transition}
(q,\gamma) \xrightarrow{a} (q',\gamma')
\end{equation}
when the PDA can consume $a\in\Sigma\cup\{\varepsilon\}$ and move from
$(q,\gamma)$ to $(q',\gamma')$. The language on $\Sigma$ accepted by $\mathcal{P}$ is denoted as $\mathcal L(\mathcal P)$.

\paragraph{Token-level configurations.}
LLMs generate tokens from a vocabulary $V$, whereas the PDA consumes terminals
from $\Sigma$. A token may represent a full terminal, several terminals, a strict
prefix of a terminal, or a string crossing terminal boundaries. We therefore
track an \emph{extended configuration}
$(q,\gamma,\rho)\in Q\times\Gamma^*\times\Sigma^*$, where $\rho$ is the open
prefix of a terminal currently being matched; $\rho=\varepsilon$ means that no
terminal is open. Let $\mathcal E=Q\times\Gamma^*\times\Sigma^*$. Token
consumption is represented as a transition function
\begin{equation}
\label{eq:consume-function}
\mathrm{Consume}:\mathcal E\times V\to 2^{\mathcal E},
\end{equation}
where $\mathrm{Consume}((q,\gamma,\rho),t)$ contains all extended configurations
reachable by matching the decoded token string against CFG terminals and
applying the induced PDA transitions.

\paragraph{Reachability and token distance.}
For a stack bound $H$, let
$\mathcal S_H\subseteq Q\times\Gamma^{\leq H}$ be a bounded pushdown summary.
Reachability in this summary follows the symbolic view of pushdown systems, in
which predecessor sets can be computed by $\mathrm{Pre}^*$-style saturation
procedures~\citep{pushdown_systems}. We write
$\mathrm{Reach}_H(q,\gamma)$ when $(q,\gamma)$ can reach an accepting
configuration within $\mathcal S_H$.

To associate summarized configurations with a token-level distance estimate, we
turn the PDA into a weighted pushdown system. For each PDA transition consuming
a terminal $a$, we assign the weight $w(a)=$ the minimum number of LLM tokens required to generate $a$.
If a terminal is specified by a regular expression $R$, we first identify a
shortest string $s_R\in L(R)$ and compute the minimum tokenization length of
$s_R$ using the tokenizer vocabulary. Epsilon
transitions are assigned weight zero. Under a stack-height bound $H$, we then
apply the standard weighted $\mathrm{Pre}^*$ algorithm to compute the minimum
accumulated token cost of reaching an accepting configuration. We denote the
resulting quantity by
\begin{equation}
\label{eq:distance-setup}
d_H:\mathcal S_H\to\mathbb N\cup\{\infty\}.
\end{equation}
In the exact weighted bounded system, $d_H(q,\gamma)$ coincides with the true
minimum token cost $d_H^*(q,\gamma)$. More generally, when conservative
approximations are used in the construction of the terminal weights, we require
the estimate to satisfy
\begin{equation}
\label{eq:upper-bound}
d_H(q,\gamma)\geq d_H^*(q,\gamma).
\end{equation}
Thus, if the remaining token budget is $K$ and
$d_H(q,\gamma)>K$, no accepting continuation of length at most $K$ exists
within the bounded summary. The resulting estimate is therefore used for
budget-aware pruning in Equation~\eqref{eq:upper-bound}.

\paragraph{Objective.}
We formulate CFG-constrained decoding as the problem of producing the
highest-probability sequence accepted by the PDA within a finite token budget
$T$:
\begin{equation}
\label{eq:constrained-objective}
y^\star
\in
\arg\max_{\substack{
y\in\mathcal L(\mathcal P)\\
|y|\leq T
}}
\sum_{i=1}^{|y|}
\log p_\theta(y_i\mid y_{<i}).
\end{equation}
Exact optimization of this sequence-level objective is intractable in general:
for autoregressive models, exact maximum-a-posteriori
decoding is NP-hard, even under simple unary constraints~\citep{pachet2026hiddenbiases}.
Accordingly, \mthdname{} treats Eq.~\eqref{eq:constrained-objective} as an
approximate search objective, guided by reachability and distance estimates.

\section{Proposed Method}
\label{sec:method}

\subsection{Overview}
\label{sec:method-overview}
\mthdname{} is a two-phase decoder for LLM generation under CFG constraints
represented as PDAs. Offline, it builds the bounded summary $\mathcal S_H$ with
reachability labels $\mathrm{Reach  }_H$ and token-distance estimates $d_H$.
Online, our beam search, with a width of $M$, maintains extended configurations $(q,\gamma,\rho)$ and uses these labels to validate, prune, and rank token candidates. Algorithm~\ref{alg:method-overview} describes the online decoding loop, Figure~\ref{fig:neurons_colored} summarizes the full pipeline, and a practical example is given in Appendix~\ref{practical_example}. Each beam
stores a prefix $y$, a score $\ell$, and an extended configuration $(q,\gamma,\rho)$.
Given a total token budget $T$, at step $k$, with the remaining horizon $K=T-k$, the decoder builds a candidate set,
executes candidates through the PDA, filters out unreachable or over-budget
successors, and keeps the top $M$ scored extensions. Here $\mathrm{Accept}(q,\gamma,\rho)$ means that the PDA acceptance condition holds and no terminal is open, i.e., $\rho=\varepsilon$.

\begin{algorithm}[t]
\normalsize
\DontPrintSemicolon
\caption{\mthdname{} decoding}
\label{alg:method-overview}

\KwIn{LLM $p_\theta$, PDA $\mathcal P$, summary $\mathcal S_H$, distance $d_H$, max length $T$, beam width $M$}
\KwOut{Best completed output $y$ found}

\BlankLine
$\mathcal B \gets \{(\varepsilon,0,q_0,z_0,\varepsilon)\}$\;

\BlankLine
\For{$k=0$ \KwTo $T-1$}{
    $K \gets T-k$\\
    $\mathcal B' \gets \emptyset$\;

    \BlankLine
    \ForEach{$(y,\ell,q,\gamma,\rho)\in\mathcal B$}{
        \If{$\mathrm{Accept}(q,\gamma,\rho)$}{
            add $(y,\ell,q,\gamma,\rho)$ to $\mathcal B'$ \\
            \textbf{continue}\;
        }

        \BlankLine
        form $T_k^{\mathrm{cand}}$ by Eq.~\eqref{eq:candidate-set}\;

        \BlankLine
        \ForEach{$t\in T_k^{\mathrm{cand}}$}{
            compute $C_t$ by Eq.~\eqref{eq:successor-set}, \;

            \BlankLine
            \ForEach{$(q',\gamma',\rho')\in C_t$}{
                score using Eq.~\eqref{eq:beam-score}\;
                add $(yt,s',q',\gamma',\rho')$ to $\mathcal B'$\;
            }
        }
    }

    \BlankLine
    $\mathcal B \gets$ top-$M$ beams in $\mathcal B'$\;
}

\BlankLine
\Return best accepting beam in $\mathcal B$
\;
\end{algorithm}

\paragraph{Offline and on-the-fly computation.}
 Offline, before decoding begins, the grammar is translated to a PDA and \mthdname{} caches $\mathcal S_H$ and $d_H$ in a configuration database reused across prompts. Online, most queries are answered by lookup, while extended configurations absent from the configuration database are computed on the fly and stored in the configuration database itself, whilst those out of bounds or unreachable are discarded. Thus, repeated use of the same grammar amortizes the cost of the summary, while on-the-fly computation covers configurations not encountered during precomputation.

\subsection{Candidate Tokens}
\label{sec:candidate-tokens}

At each step, \mthdname{} combines model-led and automaton-led proposals. Let
$\mathrm{TopK}_k(y)$ be the model top-$k$ tokens for the current prefix $y$.
The PDA also proposes tokens compatible with locally admissible continuations:
Let $T_{\mathrm{prop}}(q,\gamma,\rho)$ be the set of tokens whose decoded string
is compatible with at least one admissible continuation from the extended
configuration $(q,\gamma,\rho)$:
\begin{equation}
\label{eq:token-proposal}
T_{\mathrm{prop}}(q,\gamma,\rho)
\subseteq V .
\end{equation}
For a given configuration \((q,\gamma,\rho)\), the automaton-proposed tokens are obtained by considering every transition \((q, a, \gamma_{\text{top}} , q', \gamma') \in \Delta \) for which \(\rho\) is a prefix of \(a\), where $\gamma_{\text{top}}$ is the top element of $\gamma$. Let \(a[|\rho|:]\) be the remaining suffix after removing \(\rho\); if non‑empty, its first token (according to the model’s tokenizer) is included in \(T_{\mathrm{prop}}(q,\gamma,\rho)\).

\[
T_{\mathrm{prop}}(q,\gamma,\rho) \;=
\]
\[
\left\{
\operatorname{First}\bigl(\operatorname{tok}(a[|\rho|:])\bigr)
\;\middle|\;
\begin{aligned}
&  (q, a, \gamma_{\text{top}} , q', \gamma')\in\Delta,\\
& \rho \text{ is a prefix of } a,\\
& a[|\rho|:] \neq \varepsilon
\end{aligned}
\right\}
\]
where \(\operatorname{First}(t_1t_2\ldots)=t_1\) and \(\operatorname{tok}(\cdot)\) are the model’s tokenization functions.
The resulting candidate set is
\begin{equation}
\label{eq:candidate-set}
T_k^{\mathrm{cand}}
=
\mathrm{TopK}_k(y)
\cup
T_{\mathrm{prop}}(q,\gamma,\rho).
\end{equation}
Thus, high-probability model tokens are checked first, while automaton proposals
recover structurally useful tokens that may fall outside the model top-$M$.

\subsection{Successor Filtering}
\label{sec:successor-filtering}

For each candidate token $t$, we compute a viable successor set
\begin{equation}
\label{eq:successor-set}
C_t =
\mathcal C_{B,H}(t,e,K)
\subseteq
\mathrm{Consume}(e,t),
\end{equation}
where $e= (q,\gamma,\rho)$ and $\mathcal C_{B,H}$ explore at most $B$ successors and retain only those whose current distance estimate witnesses a completion within the remaining token horizon:
\begin{equation}
\label{eq:horizon-filter}
d_H(q',\gamma') \leq K-1.
\end{equation}
Equivalently, successors with $d_H(q',\gamma')>K-1$ are discarded.
For every token with $C_t\neq\emptyset$, its lookahead distance is
\begin{equation}
\label{eq:token-distance}
D_{B,H}(t,e)
=
\min_{(q',\gamma',\rho')\in C_t}
d_H(q',\gamma').
\end{equation}
Otherwise, the token $t$ is discarded.
\subsection{Distance-Guided Scoring}
\label{sec:distance-guided-scoring}

In addition to pruning, distance estimates are used to promote tokens that provide progress towards acceptance. Let
$z_k(t)$ be the model logit for token $t$ at step $k$, and define
\begin{equation}
\label{eq:viable-token-set}
T_k^{\mathrm{viab}}
=
\{u\in T_k^{\mathrm{cand}} : C_u\neq\emptyset\}.
\end{equation}
Let $z_k^{\max}=\max_{u\in T_k^{\mathrm{viab}}}z_k(u)$. For a viable token $t$,
we compute
\begin{equation}
\label{eq:pressure-ratio}
r_k(t)
=
\frac{D_{B,H}(t)}{\max(1,K-1)}
\end{equation}
and the ramping coefficient
\begin{equation}
\label{eq:alpha-ramp}
\alpha_k(t)
=
\alpha_{\min}
+
(1-\alpha_{\min})
\min\{1,r_k(t)\}.
\end{equation}
Tokens selected for promotion are pushed toward the best viable logit:
\begin{equation}
\label{eq:ramping-push-up}
\widetilde z_k(t)
=
(1-\alpha_k(t))z_k(t)
+
\alpha_k(t)z_k^{\max}.
\end{equation}
The successor beam score is then
\begin{equation}
\label{eq:beam-score}
s'
=
\ell+\log\mathrm{softmax}(\widetilde z_k)(t).
\end{equation}

\subsection{Soundness and Approximation}
\label{sec:soundness-approximation}
We distinguish syntactic soundness from search completeness. Syntactic soundness requires that every returned output is accepted by the target PDA, whereas search completeness requires exploring all valid completions. \mthdname{} is syntactically sound and preserves at least one certified completion whenever $d_H(e_0)\leq T$, though it is not search-complete because of bounded successor exploration and finite beam width.
\begin{theorem}[Syntactic soundness and completion]
\label{thm:soundness}
Let $T$ be the token budget and let $B,M\geq 1$.
Every output returned by \mthdname{} is accepted by $\mathcal P$.
Moreover, if the initial configuration $e_0=(q_0,z_0,\varepsilon)$ satisfies
$d_H(e_0)\leq T$, then \mthdname{} returns at least one accepted output.
\end{theorem}

\begin{proof}
Syntactic soundness follows directly from the construction of the algorithm:
a successor is retained only if it belongs to
$\mathrm{Consume}(e,t)$, and the algorithm returns only an accepting beam.

We now prove completion by induction over the decoding step $k$. Let
\begin{equation}
\label{eq:induction-hypothesis}
\begin{aligned}
\mathcal H_k:\quad
&\exists (y,\ell,e)\in\mathcal B_k \text{ such that}\\
&\mathrm{Accept}(e)\ \lor\ d_H(e)\leq T-k .
\end{aligned}
\end{equation}
\textbf{Base case ($k=0$).}
Initially, $\mathcal B_0$ contains $e_0$, and
$d_H(e_0)\leq T$ by assumption. Hence $\mathcal H_0$ holds.
\textbf{Induction step.}
Assume $\mathcal H_k$ holds and let $K=T-k$.
If its witnessing beam is accepting, it is copied unchanged to the next
candidate pool.

Otherwise, its configuration $e$ satisfies $d_H(e)\leq K$.
Since $d_H$ is computed constructively, it provides a witness continuation
toward acceptance. The first token of this witness belongs to
$T_{\mathrm{prop}}(e)$ and hence to $T_k^{\mathrm{cand}}$, independently
of the model top-$k$ truncation. The successor exploration preserves a
corresponding witness successor $e'$ with
$d_H(e')\leq K-1$; therefore $B=1$ is sufficient.

Thus, the candidate pool contains at least one accepting or viable beam.
Since non-accepting candidates are added only if they satisfy the horizon
filter, top-$M$ selection with $M\geq1$ cannot remove all such beams.
Hence $\mathcal H_{k+1}$ holds.

By induction, $\mathcal H_T$ holds. At $k=T$, no non-accepting beam can
have a positive remaining distance, so at least one accepting beam remains
and is returned.
\end{proof}

The condition $d_H(e_0)\leq T$ means that the bounded summary certifies an accepting continuation within the available budget. If it does not hold, \mthdname{} reports that completion cannot be certified under the current budget and bound $H$, rather than returning an invalid output. Since $d_H$ is an upper-bound estimate, this does not imply that no shorter accepting run exists outside the bounded summary.

\section{Experiments} \label{sec:experiments}
\begin{table*}[t]
\centering
\footnotesize
\setlength{\tabcolsep}{1.5pt}
\renewcommand{\arraystretch}{1.08}
\begin{tabular}{lcccccccccccc}
\toprule
& \multicolumn{4}{c}{Llama-3.1-8B-Instruct} 
& \multicolumn{4}{c}{Llama-3.2-3B-Instruct} 
& \multicolumn{4}{c}{Qwen2.5-7B-Instruct} \\
\cmidrule(lr){2-5}
\cmidrule(lr){6-9}
\cmidrule(lr){10-13}
Method 
& Syn. & Schema & Time & Perp.
& Syn. & Schema & Time & Perp.
& Syn. & Schema & Time & Perp. \\
\midrule

Base LM
& 63.7{\scriptsize$\pm$15.9} & 57.1{\scriptsize$\pm$14.4} & \underline{1.74}{\scriptsize$\pm$0.0} & 1.05
& 56.4{\scriptsize$\pm$9.8} & 49.4{\scriptsize$\pm$8.7} & \textbf{0.96}{\scriptsize$\pm$0.1} & \textbf{1.06}
& 93.6{\scriptsize$\pm$0.8} & 87.4{\scriptsize$\pm$0.8} & \textbf{1.47}{\scriptsize$\pm$0.0} & \textbf{1.01} \\

Sample-Verify
& 92.0{\scriptsize$\pm$0.0} & 80.9{\scriptsize$\pm$0.9} & 17.38{\scriptsize$\pm$0.0} & 1.08
& 80.0{\scriptsize$\pm$0.0} & 67.1{\scriptsize$\pm$1.9} & 9.65{\scriptsize$\pm$0.0} & 1.22
& \underline{96.0}{\scriptsize$\pm$0.0} & 88.3{\scriptsize$\pm$0.5} & 14.71{\scriptsize$\pm$0.0} & 1.04 \\

Guidance
& \underline{96.0} & 89.0 & 2.88 & 1.08
& \underline{93.0} & 84.0 & 2.76 & 1.12
& 92.0 & 87.0 & 2.85 & \underline{1.02} \\

Outlines 
& 94.0 & \underline{90.0} & 3.49 & 1.07
& \underline{93.0} & \underline{90.0} & 2.04 & 1.12
& 95.0 & \underline{89.0} & 3.62 & \underline{1.02} \\

XGrammar
& 93.0 & 85.0 & \textbf{1.57} & 1.45
& 95.0 & 88.0 & \textbf{0.96} & 1.38
& 93.0 & 85.0 & \underline{1.60} & 1.80 \\

SynCode
& 88.0 & 79.0 & 3.05 & \textbf{1.02}
& 88.0 & 77.0 & 2.12 & 1.13
& 94.0 & 88.0 & 1.50 & \textbf{1.01} \\

GenLM (NP=2)
& 90.8{\scriptsize$\pm$1.9} & 86.6{\scriptsize$\pm$1.8} & 3.77{\scriptsize$\pm$1.9} & 1.05
& 91.4{\scriptsize$\pm$1.3} & 87.5{\scriptsize$\pm$1.4} & 3.93{\scriptsize$\pm$1.1} & \underline{1.09}
& 93.6{\scriptsize$\pm$0.5} & \underline{89.0}{\scriptsize$\pm$0.0} & 2.93{\scriptsize$\pm$1.9} & \textbf{1.01} \\

GenLM (NP=5)
& 91.6{\scriptsize$\pm$1.1} & 87.2{\scriptsize$\pm$1.3} & 5.59{\scriptsize$\pm$0.3} & \underline{1.03}
& \underline{93.0}{\scriptsize$\pm$1.0} & \underline{88.4}{\scriptsize$\pm$0.9} & 7.36{\scriptsize$\pm$0.4} & \underline{1.09}
& 93.8{\scriptsize$\pm$0.4} & \underline{89.0}{\scriptsize$\pm$0.0} & 4.07{\scriptsize$\pm$0.2} & \textbf{1.01} \\

$\triangleright$ \mthdname{}
& \textbf{100.0} & \textbf{100.0} & 3.96 & 1.13
& \textbf{100.0} & \textbf{100.0} & 3.46 & 1.14
& \textbf{100.0} & \textbf{100.0} & 4.05 & 1.11 \\

\bottomrule
\end{tabular}
\caption{
JSON generation results across base models. Syn. denotes syntactic validity and Schema denotes schema validity. Best results within each model are shown in bold and second-best results are underlined. \model is ran with $\alpha=0.25$ and $2$ beams, and GenLM is ran with default configuration using $2$ and $5$ number of particles (NP). Syncode has two modes: Strict and Mask, here we only report Strict as it outperforms its counterpart all metrics. }
\label{tab:json-results}
\end{table*}
\begin{table*}[t]
\centering
\footnotesize
\setlength{\tabcolsep}{1.5pt}
\renewcommand{\arraystretch}{1.08}
\begin{tabular}{lcccccccccccc}
\toprule
& \multicolumn{4}{c}{Llama-3.1-8B-Instruct} 
& \multicolumn{4}{c}{Llama-3.2-3B-Instruct} 
& \multicolumn{4}{c}{Qwen2.5-7B-Instruct} \\
\cmidrule(lr){2-5}
\cmidrule(lr){6-9}
\cmidrule(lr){10-13}
Method
& Ex. & Syn. & Time & Perp.
& Ex. & Syn. & Time & Perp.
& Ex. & Syn. & Time & Perp. \\
\midrule

Base LM
& 53.8{\scriptsize$\pm$4.7} & 92.9{\scriptsize$\pm$7.9} & \textbf{0.87}{\scriptsize$\pm$0.0} & 1.61
& 46.3{\scriptsize$\pm$4.0} & 96.2{\scriptsize$\pm$6.7} & \textbf{0.53}{\scriptsize$\pm$0.0} & 2.16
& 61.4{\scriptsize$\pm$0.6} & 99.2{\scriptsize$\pm$0.2} & \textbf{0.88}{\scriptsize$\pm$0.0} & \underline{3.48} \\

Sample-Verify
& 57.5{\scriptsize$\pm$1.1} & \textbf{100.0}{\scriptsize$\pm$0.0} & 8.90{\scriptsize$\pm$0.0} & \underline{1.56}
& 47.8{\scriptsize$\pm$0.7} & 99.1{\scriptsize$\pm$0.3} & 5.29{\scriptsize$\pm$0.0} & 3.50
& \underline{61.5}{\scriptsize$\pm$0.4} & \underline{99.7}{\scriptsize$\pm$0.1} & 8.83{\scriptsize$\pm$0.0} & 10.38 \\

SynCode
& 43.2 & 90.7 & \underline{3.31} & 3.70
& 39.6 & 98.1 & 2.01 & 2.47
& 43.2 & 90.7 & \underline{3.31} & 3.70 \\

GenLM (NP=5,5,4)
& \underline{59.2}{\scriptsize$\pm$0.9} & \underline{99.9}{\scriptsize$\pm$0.1} & 10.74{\scriptsize$\pm$8.4} & 1.67
& \underline{49.4}{\scriptsize$\pm$0.7} & \underline{99.9}{\scriptsize$\pm$0.1} & 4.17{\scriptsize$\pm$0.3} & \textbf{1.38}
& 51.9{\scriptsize$\pm$0.4} & 86.1{\scriptsize$\pm$0.6} & 6.77{\scriptsize$\pm$0.1} & 3.78 \\

$\triangleright$ \mthdname{}
& \textbf{61.1} & \textbf{100.0} & 9.92 & \textbf{1.22}
& \textbf{51.5} & \textbf{100.0} & 6.25 & \underline{1.43}
& \textbf{62.8} & \textbf{100.0} & 13.03 & \textbf{2.10} \\

\bottomrule
\end{tabular}
\caption{
SQL generation results across base models. Ex. denotes execution accuracy and
Syn. denotes syntactic correctness. Best results within each model are shown in
bold and second-best results are underlined. GenLM uses NP=5 for Llama models
and NP=4 for Qwen. \mthdname{} uses beam width $4$; $\alpha$ is set to $0.25$,
$0.75$, and $0.50$ for Llama-3.1-8B, Llama-3.2-3B, and Qwen2.5-7B,
respectively. 
}
\label{tab:sql-results}
\end{table*}

\begin{table*}[t]
\centering
\footnotesize
\setlength{\tabcolsep}{2pt}
\renewcommand{\arraystretch}{1.08}
\begin{tabular}{lcccccccccccc}
\toprule
& \multicolumn{4}{c}{Llama-3.1-8B-Instruct} 
& \multicolumn{4}{c}{Llama-3.2-3B-Instruct} 
& \multicolumn{4}{c}{Qwen2.5-7B-Instruct} \\
\cmidrule(lr){2-5}
\cmidrule(lr){6-9}
\cmidrule(lr){10-13}
Method
& Acc. & Syn. & Time & Perp.
& Acc. & Syn. & Time & Perp.
& Acc. & Syn. & Time & Perp. \\
\midrule

Base LM
& 22.1{\scriptsize$\pm$3.3} & 98.8{\scriptsize$\pm$0.4} & \textbf{0.29}{\scriptsize$\pm$0.0} & 1.31
& 9.6{\scriptsize$\pm$3.3} & 94.6{\scriptsize$\pm$3.1} & \textbf{0.17}{\scriptsize$\pm$0.0} & 1.57
& 30.9{\scriptsize$\pm$0.5} & 98.4{\scriptsize$\pm$0.4} & \textbf{0.29}{\scriptsize$\pm$0.0} & 1.61 \\

Sample-Verify
& 22.2 {\scriptsize$\pm$0.3} & \textbf{100.0} {\scriptsize$\pm$0.0} & 2.86 {\scriptsize$\pm$0.0} & 1.31
& 9.3 {\scriptsize$\pm$0.2} & \underline{99.9}{\scriptsize$\pm$0.0} & 1.64{\scriptsize$\pm$0.0} & 1.56
& 30.6 {\scriptsize$\pm$0.2} & \underline{99.7} {\scriptsize$\pm$0.0} & 2.89 {\scriptsize$\pm$0.0} & \textbf{1.07} \\
SynCode 
& \underline{24.4} & \underline{99.9} & \underline{0.30} & \textbf{1.21} 
& \underline{15.2} & \underline{99.9} & \underline{0.19} & \underline{1.36} 
& \underline{31.0} & \textbf{100.0} & \underline{0.31} & \underline{1.09} \\

$\triangleright$ \mthdname{}
& \textbf{29.6} & \textbf{100.0} & 1.69 & \underline{1.23}
& \textbf{22.1} & \textbf{100.0} & 0.78 & \textbf{1.27}
& \textbf{31.6} & \textbf{100.0} & 1.53 & \underline{1.09} \\

\bottomrule
\end{tabular}
\caption{
LTL generation results across base models. Acc. denotes task accuracy and Syn.
denotes syntactic correctness. Best results within each model are shown in bold.
\mthdname{} uses beam width $4$; $\alpha$ is set to $0.75$, $0.50$, and $0.50$
for Llama-3.1-8B, Llama-3.2-3B, and Qwen2.5-7B, respectively.
}
\label{tab:ltl-results}
\end{table*}
\paragraph{Models.} We evaluate \mthdname{} on structured generation tasks requiring syntactic validity under CFG using three instruction-tuned large language models that differ in size, family, and architecture: \textsc{Llama-3.1-8B-Instruct} \cite{grattafiori2024llama3herdmodels}, \textsc{Llama-3.2-3B-Instruct} \cite{meta_llama_3_2_3b_instruct}, and \textsc{Qwen2.5-7B-Instruct}.
Reporting results on
all three lets us assess robustness across base architectures rather than tuning to a single model.

\paragraph{Benchmarks.} We consider three structured generation domains: JSON, SQL, and Linear Temporal Logic (LTL).
For JSON, we use the \texttt{json-mode-eval} dataset \cite{nousresearch2024jsonmodeeval}, which consists of 100 examples, and requires generating valid JSON objects adhering to a given schema, testing the model's ability to produce properly nested structures and correct data types. 
For SQL, we employ the Spider dataset \cite{spider}, a cross-domain text-to-SQL benchmark where each example pairs a natural language question with a database schema, and the model must generate a syntactically valid SQL query that correctly retrieves the answer. We use the official validation split, which contains 1,034 examples.
For LTL, we adopt the drone planning task from \citet{ltldataset}, using their golden dataset comprising 6,185 examples. The model must generate LTL formulas that specify planning goals for drones, ensuring that each output strictly follows LTL syntax.

\paragraph{Baselines.} We include two simple baselines: \textsc{Base LM} and \textsc{Sample-Verify}. \textsc{Base LM} runs the unconstrained base model with a temperature of 0.8 for 10 independent generations per input (different random seeds). \textsc{Sample-Verify} uses the same 10 raw generations, discards syntactically invalid ones, and then, for each of the 10 fixed seeds, randomly selects one surviving generation per seed (if any) to form 10 final runs. This follows the protocol of \citet{lipkin2025fastcontrolledgenerationlanguage}.
For comparison, we consider all available open-source constrained decoding methods whose public implementations support by default the grammar required by the task. Concretely:\\
\textbf{JSON}: Guidance \cite{lundberg2024guidance}, Outlines \cite{willard2023efficientguidedgenerationlarge}, XGrammar \cite{dong2025xgrammarflexibleefficientstructured}, SynCode \cite{ugare2024syncodellmgenerationgrammar}, and GenLM \cite{loula2025syntacticsemanticcontrollarge,lipkin2025fastcontrolledgenerationlanguage}.\\
\textbf{SQL}: SynCode and GenLM. While XGrammar technically supports SQL it took up to 20 seconds per generated token and was therefore discarded.\\
\textbf{LTL}: While LTL is a CFG, currently, none of the public implementations natively support LTL grammar.\\
The following methods are excluded due to lack of runnable public code: DOMINO\footnote{A public repository exists but contains no executable implementation at the time of writing.} \cite{beurerkellner2024guidingllmsrightway}, GreatGramma \cite{park2025flexibleefficientgrammarconstraineddecoding}, and SEM-CTRL \cite{semctrl}. Additional approaches (GeLaTo \cite{zhang2023tractablecontrolautoregressivelanguage}, Ctrl-G \cite{zhang2024adaptablelogicalcontrollarge}, ABS \cite{collura2025absenforcingconstraintsatisfaction}) are restricted to regular grammars; Pre$^3$\cite{Chen_2025} handles only deterministic CFGs; and Type-constr. \cite{M_ndler_2025} works only for TypeScript. For each baseline, we report the mean and standard deviation over multiple seeds; deterministic methods are reported with a single run.

\paragraph{Choice of the Token Budget $T$.}
We uniformly set the token budget to $T=120$ across all three tasks. This is more than sufficient for the target outputs and is comparable to budgets used in prior constrained decoding work \cite{zhang2023tractablecontrolautoregressivelanguage, lipkin2025fastcontrolledgenerationlanguage}. Moreover, as shown in Appendix~\ref{saturation}, the performance of existing local constrained decoding methods largely saturates for $T>512$. Increasing the generation budget allows these methods to produce longer sequences, but does not resolve the underlying constraint-satisfaction problem. In particular, local decoding may recognize that opening a new structure, such as a nested JSON object or array, is grammatically valid, yet lacks global lookahead or an explicit signal for when to stop opening structures and begin closing them. It may therefore keep extending locally valid contexts without progressing toward an accepting configuration. This explains why increasing the budget from $T=120$ to $T=4096$ (up to $40\times$) yields little or no improvement in JSON schema satisfaction, even when syntax validity reaches $100\%$. Our distance-from-accepting-configuration component addresses this myopia by guiding decoding toward configurations from which acceptance remains reachable within the available budget, providing the missing mechanism for timely closure and reliable constraint satisfaction.

\paragraph{Metrics.} For each domain, we measure both constraint satisfaction and generation quality.  

\noindent\textbf{JSON.} We report \textit{syntactic validity} (fraction of outputs parsable as JSON), \textit{schema validity} (fraction conforming to the required schema, e.g., correct field names and types), \textit{perplexity} (mean token-level perplexity of the generated sequences under the base LM, measuring fluency), and \textit{average decoding time} (seconds per output, capturing computational overhead).

\noindent\textbf{SQL.} We report \textit{syntactic correctness} (the percentage of queries that parse according to the SQL grammar), \textit{execution accuracy} (the percentage that returns the expected result when run on the target database), \textit{perplexity}, and \textit{decoding time}.

\noindent\textbf{LTL.} We report \textit{syntactic correctness} (parser acceptance for LTL formulas), \textit{task accuracy} (the percentage of generated LTL formulas whose corresponding B\"uchi automaton is equivalent to that of the ground-truth formula, verifying semantic correctness), \textit{perplexity}, and \textit{decoding time}.

\label{sec:experiments-json}

\paragraph{Grammar preprocessing cost.}
 The standard CFG-to-PDA construction is linear in the size of the grammar and introduces only modest overhead; for the more complex syntax, SQL, it took 0.67 seconds. Summary precomputation cost, in contrast, can vary significantly depending on the grammar complexity. For example, it requires 3h 40min for SQL with $H=13$ (supporting up to three nested queries), whereas for LTL with $H=50$ (supporting up to 50 nested parentheses), it required 0.42s. For schema-specific JSON grammars, both conversion and summary precomputation took an average of 0.29s ($\pm 0.11$s). Importantly, both of these steps are grammar-dependent rather than dataset-dependent and can therefore be reused across all tasks involving the same grammar.

\paragraph{Main findings.}
Across all three domains and base models, \mthdname{} is the only method that
achieves $100\%$ syntactic correctness in every setting. On JSON, this also
corresponds to $100\%$ schema validity across all models. On SQL and LTL,
\mthdname{} improves task-specific accuracy while preserving syntactic validity.
The cost is a moderate decoding-time overhead relative to lightweight local
masking methods, but it is comparable to that of repeated sample-and-verify, which achieves lower performance in all domains.

\paragraph{Ablation studies.}
We assess the contributions of the two main components of \mthdname{}, Beam Search and Distance-Guided Scoring (DGS), through ablations on SQL and LTL, where task-quality metrics are available (Appendix~\ref{ablation}). Across all three models, syntactic correctness remains at $100\%$ in every configuration, confirming that syntax guarantees are independent of these components. Beam Search provides the largest gains, substantially improving execution accuracy on SQL and semantic accuracy on LTL while reducing perplexity. DGS adds smaller but consistent improvements in task accuracy with negligible decoding overhead. Overall, combining Beam Search and DGS yields the best trade-off between efficiency and quality.

\subsection{JSON Generation}

Table~\ref{tab:json-results} reports results on \texttt{json-mode-eval}. \mthdname{} is the only method that achieves perfect syntactic and schema validity across all three base models. Other constrained decoders generally improve over the Base LM on the two Llama models, but their gains are less consistent on Qwen2.5-7B, where the unconstrained model is already strong. \mthdname{} maintains 100\% validity in all settings, at a moderate runtime cost: it is slower than lightweight masking-based methods, but faster than Sample-Verify and competitive with GenLM at higher particle counts. Its perplexity remains close to the best baseline, suggesting that the guarantee does not come at the cost of substantially lower model likelihood.

\subsection{SQL Generation}

Table~\ref{tab:sql-results} reports results on Spider. \mthdname{} achieves
perfect syntactic correctness across all three base models and the highest
execution accuracy in every setting. Sample-Verify also reaches perfect or
near-perfect syntax but does not improve execution accuracy to the same extent,
showing that syntactic filtering alone is insufficient for this task. SynCode
variants underperform the Base LM in execution accuracy, while GenLM is
competitive on the Llama models but degrades on Qwen2.5-7B-Instruct. In contrast, \mthdname{} consistently improves both syntactic correctness and execution accuracy across architectures. Its runtime is comparable to Sample-Verify and GenLM on the Llama models, though it is slower on Qwen, and its perplexity remains the best or second-best in all settings. Overall, \mthdname{} provides the strongest trade-off between syntactic guarantees, execution accuracy, and decoding cost.

\subsection{LTL Generation}
Table~\ref{tab:ltl-results} reports results on the LTL dataset. \mthdname{} consistently achieves $100\%$ syntactic correctness and the highest task accuracy across all three models. SynCode improves task accuracy over the Base LM in every setting, but provides smaller gains and does not consistently guarantee syntactic correctness, reaching $100\%$ only for Qwen2.5-7B. In particular, for Llama-3.1-8B, its syntactic correctness is slightly below that of Sample-Verify. Overall, \mthdname{} provides the strongest combination of syntax guarantees and task accuracy, at the cost of a moderate decoding-time overhead.

\section{Conclusion}
We have introduced \mthdname, a distance-guided decoding framework for guaranteed context-free grammar compliance in large language models. 

Unlike prior grammar-constrained decoders that rely solely on local prefix feasibility, SWYB augments token-level validation with bounded pushdown summaries and upper-bound distance estimates to acceptance. This enables budget-aware pruning and reranking: configurations whose estimated completion cost exceeds the remaining token horizon are discarded; conversely, our beam search softly promotes tokens that make measurable structural progress toward acceptance.

Experiments on JSON, SQL, and LTL generation demonstrate that SWYB consistently achieves perfect syntactic validity across three diverse base models while also improving task-specific accuracy over existing methods, including Sample-Verify, SynCode, XGrammar, and GenLM. The method maintains competitive perplexity and decoding time, offering a favorable trade-off between guarantee and efficiency.

\section*{Limitations}
Among the limitations are the stack height bound \(H\) and the
exploration budget \(B\). The parameter \(H\) determines the maximum
stack height considered during offline precomputation. If \(H\) is set
too low, some reachable configurations may be missing from the
bounded summary; in that case, the decoder may need to compute
distances online for configurations not present in the precomputed
table. Although such online computation is not prohibitively
expensive, it is slower than a simple lookup. The exploration budget \(B\) controls
how many intermediate PDA configurations are explored. A large
\(B\) increases runtime but enables more detailed lookahead. The
impact of \(B\) depends on the complexity of the PDA: for small
grammars such as LTL, exhaustive exploration is perfectly feasible
and incurs little overhead. Importantly, even with a minimal budget
\(B = 1\) and a sufficiently large \(H\), the algorithm remains
sound because soundness only requires knowing that at least one
valid completion exists within the remaining token budget;
detailed enumeration of all possible successors is not needed.
The distance estimate $d_H$ is a constructive upper-bound estimate within the
bounded summary. A small value certifies that a completion within the remaining
token budget is known, but a large value does not prove that no shorter
completion exists.
Finally, while \model{} empirically improves generation quality, it does
not incorporate semantic control (e.g., domain-specific
constraints). Adding such semantic guidance remains an interesting
direction for future work.
\section*{Ethics Statement}
Large Language Models (LLMs) were used in a limited manner for two specific purposes: (1) suggesting alternative phrasings for a few sentences (simple copy-editing), and (2) generating LaTeX code for tables and formatting. LLMs were not used for research design, analysis, or interpretation of results. All generated output was reviewed and verified by the authors, who assume full responsibility for the final content.
\FloatBarrier
\section*{Acknowledgments}

This work was supported by the AI for Math Fund, which is managed by Renaissance Philanthropy in partnership with founding donor XTX Markets.

This research was funded in full or in part by the Luxembourg National Research Fund (FNR, grant C23/IS/18177547/VARIANCE).

\bibliography{our}
\appendix
\label{sec:appendix}
\section{System specifications}
All the experiments were run on a machine equipped with the following infrastructure: Intel® Xeon® Silver 4416+ CPU, 20 Cores, 40 Threads, 2.00/3.90 GHz; NVidia L40S GPU, 48 GB GDDR6, 18176 CUDA Cores, 142 RT Cores, 568 Tensor Cores.

\section{Ablation studies}
\label{ablation}
The ablation results in Tables~\ref{tab:main_four} and~\ref{tab:ltl_four} show that the additional components consistently improve generation quality while preserving the syntax guarantees of the method. We evaluate the ablations on SQL and LTL, for which meaningful output-quality metrics are available, whereas JSON does not provide a comparable quality metric. Beam Search has the largest impact on execution accuracy: on SQL, using four beams increases execution accuracy from $26.0\%$ to $50.9\%$ for Llama-3.2-3B, and similar improvements are observed across all models and for LTL. This improvement arises because if the model proposes only one token at a time and this violates the PDA’s constraints, the method will select the token proposed by the PDA with the highest probability, for the LLM, this significantly reduces the quality of the output. If LLM proposes multiple tokens, there is a greater chance that it will propose a good one and avoid resorting to the PDA. This is evident from the significant difference between the first and second rows of the tables. The corresponding decrease in perplexity is also expected, since without beam search the selected tokens are more frequently determined by the PDA rather than by the LLM probability distribution. In contrast, Distance-Guided Scoring (DGS) provides a consistent, albeit smaller, improvement in execution accuracy when comparing the corresponding configurations with and without DGS. Moreover, the parameter $\alpha$ provides a direct trade-off between following the LLM and prioritizing the PDA: lower values favor the fluency of the model-generated sequence, whereas higher values favor satisfying the constraints more quickly. As guaranteed by our theorem, syntax correctness remains $100\%$ in all configurations, including those without Beam Search or DGS.

\begin{table*}[!t]
\centering
\small


\begin{tabular}{l *{3}{rrrr}}
\toprule
\textbf{Configuration} 
& \multicolumn{4}{c}{\textbf{Llama-3.2-3B}} 
& \multicolumn{4}{c}{\textbf{Llama-3.1-8B}} 
& \multicolumn{4}{c}{\textbf{Qwen2.5-7B}} \\
\cmidrule(lr){2-5} \cmidrule(lr){6-9} \cmidrule(lr){10-13}
& Exec & Syn & Perp & Time 
& Exec & Syn & Perp & Time 
& Exec & Syn & Perp & Time \\
\midrule

$\lnot$DGS, $\lnot$BS  
& 26.0 & \textbf{100} & 3.67 & \textbf{5.03} 
& 22.1 & \textbf{100} & 2.77 & \underline{8.94} 
& 22.2 & \textbf{100} & 4.11 & \underline{13.12} \\

$\lnot$DGS, 4 Beams         
& \underline{50.9} & \textbf{100} & \textbf{1.41} & 8.00 
& \underline{60.4} & \textbf{100} & \textbf{1.22} & 9.44 
& \underline{61.6} & \textbf{100} & \textbf{2.08} & 13.35 \\

DGS ($\alpha=0.5$), $\lnot$BS 
& 26.3 & \textbf{100} & 3.49 & \underline{5.04} 
& 22.2 & \textbf{100} & 2.69 & \textbf{8.83} 
& 23.8 & \textbf{100} & 4.09 & \textbf{12.77} \\

DGS ($\alpha=0.5$), 4 Beams        
& \textbf{51.2} & \textbf{100} & \textbf{1.41} & 7.75 
& \textbf{60.9} & \textbf{100} & \textbf{1.22} & 9.46 
& \textbf{62.8} & \textbf{100} & \underline{2.10} & 13.03 \\

\bottomrule
\end{tabular}

\captionof{table}{Ablation studies on SQL generation. Where DGS stands for Distance-Guided Scoring and BS for Beam Search.}
\label{tab:main_four}

\vspace{0.5em}

\begin{tabular}{l *{3}{rrrr}}
\toprule
\textbf{Configuration} 
& \multicolumn{4}{c}{\textbf{Llama-3.2-3B}} 
& \multicolumn{4}{c}{\textbf{Llama-3.1-8B}} 
& \multicolumn{4}{c}{\textbf{Qwen2.5-7B}} \\
\cmidrule(lr){2-5} \cmidrule(lr){6-9} \cmidrule(lr){10-13}
& Exec & Syn & Perp & Time 
& Exec & Syn & Perp & Time 
& Exec & Syn & Perp & Time \\
\midrule

$\lnot$DGS, $\lnot$BS 
& 16.5 & \textbf{100} & 1.43 & \underline{0.35} 
& \underline{29.2} & \textbf{100} & 1.73 & 0.63 
& 30.5 & \textbf{100} & 1.17 & \underline{0.56} \\

$\lnot$DGS, 4 Beams            
& \underline{21.8} & \textbf{100} & \underline{1.28} & 0.79 
& \underline{29.4} & \textbf{100} & 1.23 & 1.71 
& \underline{31.2} & \textbf{100} & \textbf{1.09} & 1.54 \\

DGS ($\alpha=0.5$), $\lnot$BS 
& 16.8 & \textbf{100} & 1.39 & \textbf{0.31} 
& 24.9 & \textbf{100} & 1.86 & \textbf{0.59} 
& 30.5 & \textbf{100} & 1.17 & \textbf{0.55} \\

DGS ($\alpha=0.5$), 4 Beams       
& \textbf{22.1} & \textbf{100} & \textbf{1.27} & 0.78 
& \textbf{29.6} & \textbf{100} & \textbf{1.23} & \underline{1.70} 
& \textbf{31.6} & \textbf{100} & \textbf{1.09} & 1.53 \\

\bottomrule
\end{tabular}

\captionof{table}{Ablation studies on LTL generation. Where DGS stands for Distance-Guided Scoring and BS for Beam Search.}
\label{tab:ltl_four}

\end{table*}

\section{Budget saturation problem}
\label{saturation}
To investigate whether the limited performance of existing grammar-constrained decoding methods is primarily caused by the generation budget, we repeat the JSON experiments with increasingly large values of \texttt{max\_new\_tokens}, from 512 to 4096 (up to $40\times$ our default budget). As shown in Tables~\ref{tab:json512}, \ref{tab:json1024}, \ref{tab:json2048}, and \ref{tab:json4096}, increasing the generation budget does not materially improve schema satisfaction. In particular, schema validity remains substantially below $100\%$ for all methods and models, with most methods showing little variation as the budget increases. While some methods achieve near-perfect or even $100\%$ syntactic validity, this does not translate into full schema satisfaction.
These results indicate that the failures are not primarily caused by insufficient generation length. Rather, they expose a limitation of local grammar-constrained decoding: the decoder can repeatedly select locally valid actions, such as opening nested objects or arrays, without receiving a signal that guides it toward an accepting configuration. Consequently, providing a larger token budget does not resolve the underlying constraint-satisfaction problem and can instead substantially increase decoding time, as particularly evident for SynCode across Tables~\ref{tab:json512}--\ref{tab:json4096}. Our distance-from-accepting-configuration component addresses this limitation by explicitly guiding decoding toward configurations from which acceptance can be reached, enabling reliable constraint satisfaction without relying on arbitrarily large generation budgets.

\begin{table*}[!t]
\centering
\footnotesize
\setlength{\tabcolsep}{1.5pt}
\renewcommand{\arraystretch}{1.08}
\begin{tabular}{lcccccccccccc}
\toprule
& \multicolumn{4}{c}{Llama-3.1-8B-Instruct} 
& \multicolumn{4}{c}{Llama-3.2-3B-Instruct} 
& \multicolumn{4}{c}{Qwen2.5-7B-Instruct} \\
\cmidrule(lr){2-5}
\cmidrule(lr){6-9}
\cmidrule(lr){10-13}
Method 
& Syn. & Schema & Time & Perp.
& Syn. & Schema & Time & Perp.
& Syn. & Schema & Time & Perp. \\
\midrule

Guidance
& 98.0 & 90.0 & \underline{3.63} & 1.07
& \underline{98.0} & 85.0 & \underline{3.16} & 1.10
& \underline{98.0} & 91.0 & 3.56 & \textbf{1.01 }\\

Outlines 
& \underline{99.0} & 91.0 & 3.74 & 1.07
& 95.0 & \textbf{92.0} & \textbf{1.95} & \underline{1.09}
& \textbf{100.0} & \textbf{93.0} & 3.62 & \underline{1.02} \\

SynCode
& \textbf{100.0} & 83.0 & 13.27 & \textbf{1.02}
& \textbf{99.0} & 79.0 & 7.52 & 1.11
& \textbf{100.0} & 92.0 & \textbf{1.61} & \textbf{1.01 }\\

GenLM (NP=2)
& 98.0{\scriptsize$\pm$0.0} & \underline{92.6{\scriptsize$\pm$0.7}} & \textbf{2.50{\scriptsize$\pm$1.4}} & \underline{1.04}
& 95.7{\scriptsize$\pm$0.8} & 89.3{\scriptsize$\pm$1.3} & 4.00{\scriptsize$\pm$1.1} & \textbf{1.08}
& \underline{98.0}{\scriptsize$\pm$0.0} & 92.6{\scriptsize$\pm$0.7} & \underline{2.50{\scriptsize$\pm$1.4}} & \textbf{1.01}\\

GenLM (NP=5)
& 98.0{\scriptsize$\pm$0.0} & \textbf{92.9{\scriptsize$\pm$0.3}} & 4.20{\scriptsize$\pm$0.1} & \underline{1.04}
& 97.2{\scriptsize$\pm$0.6} & \underline{90.8{\scriptsize$\pm$0.9}} & 6.10{\scriptsize$\pm$0.7} & \textbf{1.08}
& \underline{98.0}{\scriptsize$\pm$0.0} & \underline{92.9{\scriptsize$\pm$0.3}} & 4.20{\scriptsize$\pm$0.1} & \textbf{1.01 }\\

\bottomrule
\end{tabular}

\caption{Max New Tokens = 512.}
\label{tab:json512}

\footnotesize
\setlength{\tabcolsep}{1.5pt}
\renewcommand{\arraystretch}{1.08}
\begin{tabular}{lcccccccccccc}
\toprule
& \multicolumn{4}{c}{Llama-3.1-8B-Instruct} 
& \multicolumn{4}{c}{Llama-3.2-3B-Instruct} 
& \multicolumn{4}{c}{Qwen2.5-7B-Instruct} \\
\cmidrule(lr){2-5}
\cmidrule(lr){6-9}
\cmidrule(lr){10-13}
Method 
& Syn. & Schema & Time & Perp.
& Syn. & Schema & Time & Perp.
& Syn. & Schema & Time & Perp. \\
\midrule

Guidance
& 98.0 & 90.0 & \underline{3.62} & 1.07
& \underline{98.0} & 85.0 & 3.12 & 1.10
& \underline{98.0} & 91.0 & 3.61 & \textbf{1.01} \\

Outlines 
& \underline{99.0} & \textbf{91.0} & 3.70 & 1.07
& 95.0 & \underline{92.0} & \underline{1.93} & \underline{1.09}
& \textbf{100.0} & \textbf{93.0} & 3.66 & \underline{1.02} \\

SynCode
& \textbf{100.0} & 83.0 & 26.89 & \textbf{1.02}
& \textbf{99.0} & 79.0 & 15.14 & 1.11
& \textbf{100.0} & 92.0 & \textbf{1.60} & \textbf{1.01} \\

GenLM (NP=2)
& 96.3{\scriptsize$\pm$0.7} & 89.8{\scriptsize$\pm$0.9} & 5.4{\scriptsize$\pm$1.8} & \underline{1.04}
& 96.1{\scriptsize$\pm$0.9} & 90.1{\scriptsize$\pm$1.3} & 4.8{\scriptsize$\pm$1.2} & \textbf{1.08}
& \underline{98.0}{\scriptsize$\pm$0.2} & \underline{92.9{\scriptsize$\pm$0.6}} & \underline{2.6{\scriptsize$\pm$1.4}} & \textbf{1.01}\\

GenLM (NP=5)
& 96.9{\scriptsize$\pm$1.1} & \underline{90.8{\scriptsize$\pm$1.1}} & 6.8{\scriptsize$\pm$2.3} & \underline{1.04}
& 97.0{\scriptsize$\pm$0.9} & \underline{90.4{\scriptsize$\pm$1.2}} & 8.2{\scriptsize$\pm$1.3} & \textbf{1.08}
& \underline{98.0}{\scriptsize$\pm$0.2} & 92.8{\scriptsize$\pm$0.4} & 4.2{\scriptsize$\pm$0.1} & \textbf{1.01}\\

\bottomrule
\end{tabular}

\caption{Max New Tokens = 1024.}
\label{tab:json1024}

\footnotesize
\setlength{\tabcolsep}{1.5pt}
\renewcommand{\arraystretch}{1.08}
\begin{tabular}{lcccccccccccc}
\toprule
& \multicolumn{4}{c}{Llama-3.1-8B-Instruct} 
& \multicolumn{4}{c}{Llama-3.2-3B-Instruct} 
& \multicolumn{4}{c}{Qwen2.5-7B-Instruct} \\
\cmidrule(lr){2-5}
\cmidrule(lr){6-9}
\cmidrule(lr){10-13}
Method 
& Syn. & Schema & Time & Perp.
& Syn. & Schema & Time & Perp.
& Syn. & Schema & Time & Perp. \\
\midrule

Guidance
& 98.0 & 90.0 & \textbf{3.62} & 1.07
& \underline{98.0} & 85.0 & \underline{3.13} & 1.10
& 98.0 & 91.0 & 3.48 & \textbf{1.01} \\

Outlines 
& \underline{99.0} & \textbf{91.0} & \underline{3.69} & 1.07
& 95.0 & \textbf{92.0} & \textbf{1.94} & \underline{1.09}
& \underline{99.0} & 91.0 & 3.69 & \underline{1.09} \\

SynCode
& \textbf{100.0} & 83.0 & 55.84 & \textbf{1.02}
& \textbf{99.0} & 80.0 & 31.65 & 1.11
& \textbf{100.0} & 92.0 & \textbf{1.61} & \textbf{1.01} \\

GenLM (NP=2)
& 96.7{\scriptsize$\pm$0.7} & 89.9{\scriptsize$\pm$1.7} & 6.8{\scriptsize$\pm$3.0} & \underline{1.04}
& 96.0{\scriptsize$\pm$1.2} & 89.6{\scriptsize$\pm$1.3} & 7.7{\scriptsize$\pm$2.8} & \textbf{1.08}
& 98.0{\scriptsize$\pm$0.0} & \underline{92.7{\scriptsize$\pm$0.7}} & \underline{2.5{\scriptsize$\pm$1.5}} & \textbf{1.01} \\

GenLM (NP=5)
& 96.5{\scriptsize$\pm$0.5} & \underline{90.2{\scriptsize$\pm$0.8}} & 10.3{\scriptsize$\pm$5.9} & \underline{1.04}
& 97.1{\scriptsize$\pm$0.7} & \underline{90.2{\scriptsize$\pm$1.5}} & 12.7{\scriptsize$\pm$3.9} & \textbf{1.08}
& 98.0{\scriptsize$\pm$0.0} & \textbf{92.8{\scriptsize$\pm$0.6}} & 4.0{\scriptsize$\pm$0.2} & \textbf{1.01} \\

\bottomrule
\end{tabular}

\caption{Max New Tokens = 2048.}
\label{tab:json2048}

\footnotesize
\setlength{\tabcolsep}{1.5pt}
\renewcommand{\arraystretch}{1.08}
\begin{tabular}{lcccccccccccc}
\toprule
& \multicolumn{4}{c}{Llama-3.1-8B-Instruct} 
& \multicolumn{4}{c}{Llama-3.2-3B-Instruct} 
& \multicolumn{4}{c}{Qwen2.5-7B-Instruct} \\
\cmidrule(lr){2-5}
\cmidrule(lr){6-9}
\cmidrule(lr){10-13}
Method 
& Syn. & Schema & Time & Perp.
& Syn. & Schema & Time & Perp.
& Syn. & Schema & Time & Perp. \\
\midrule

Guidance
& 98.0 & \underline{90.0} & \textbf{3.57} & 1.07
& \underline{98.0} & 85.0 & \underline{3.08} & 1.11
& 98.0 & 91.0 & 3.46 & \textbf{1.01} \\

Outlines 
& \underline{99.0} & \textbf{91.0} & \underline{3.68} & 1.07
& 95.0 & \textbf{92.0} & \textbf{1.91} & \underline{1.09}
& \underline{100.0} & \textbf{93.0} & 3.60 & \underline{1.02} \\

SynCode
& \textbf{100.0} & 83.0 & 113.81 & \textbf{1.02}
& \textbf{99.0} & 79.0 & 63.47 & 1.12
& \textbf{100.0} & 92.0 & \textbf{1.60} & \textbf{1.01} \\

GenLM (NP=2)
& 96.4{\scriptsize$\pm$0.5} & \textbf{89.9{\scriptsize$\pm$0.9}} & 22.0{\scriptsize$\pm$8.9} & \underline{1.04}
& 95.5{\scriptsize$\pm$1.3} & \underline{89.8{\scriptsize$\pm$1.5}} & 20.4{\scriptsize$\pm$15.2} & \textbf{1.08}
& 98.0{\scriptsize$\pm$0.0} & 92.6{\scriptsize$\pm$0.5} & \underline{2.4{\scriptsize$\pm$1.5}} & \textbf{1.01} \\

GenLM (NP=5)
& 96.3{\scriptsize$\pm$0.5} & 89.8{\scriptsize$\pm$1.0} & 26.3{\scriptsize$\pm$19.1} & \underline{1.04}
& 96.7{\scriptsize$\pm$1.1} & \textbf{90.5{\scriptsize$\pm$0.8}} & 29.0{\scriptsize$\pm$16.7} & \textbf{1.08}
& 98.0{\scriptsize$\pm$0.0} & \underline{92.8{\scriptsize$\pm$0.4}} & 4.2{\scriptsize$\pm$0.2} & \textbf{1.01} \\

\bottomrule
\end{tabular}

\caption{Max New Tokens = 4096.}
\label{tab:json4096}

\end{table*}

\section{Practical Example}
\label{practical_example}

To illustrate how our constrained beam search algorithm works in practice, we consider a simple language: strings of balanced parentheses followed by a single `x`. The grammar is
\[
S \rightarrow \texttt{ '(' } \, S \, \texttt{ ')' } \, S \quad \mid \quad \texttt{ 'x' }
\]
Valid strings include: \texttt{x}, \texttt{()x}, \texttt{(())x}.

\subsection{Pushdown Automaton (PDA)}
This language is recognised by the following PDA with a stack:
\[
P = (Q,\Sigma,\Gamma,\Delta,q_0,Z_0,F),
\]
where $Q=\{q_0,q_f\}$, $\Sigma=\{\texttt{(},\texttt{)},\texttt{x}\}$, $\Gamma=\{Z_0,\texttt{(}\,\}$, $q_0$ is the start state, $Z_0$ is the bottom‑of‑stack marker, $F=\{q_f\}$ and the transition relation \(\Delta\) is given in Table~\ref{tab:pda}.
Intuitively, the PDA pushes a `(` for every opening parenthesis and pops one for each closing parenthesis; it moves to the accepting state $q_f$ only upon reading `x` when the stack contains exactly $Z_0$ (i.e., all parentheses are balanced). For simplicity, the vocabulary is defined as: $V=\{\text{'(', ')', 'x'}\}$ (In practice, longer tokens such as '((' may exist, but we omit them for clarity).

\subsection{Beam Search Setup}
The parameters used are: token budget (maximum steps) $T=3$, beam width $M=2$, score weighting $\alpha = 0.25$, stack height bound $H=5$.
To guide the search, we employ a distance function $d_H$ that gives the minimum number of additional tokens needed to reach an accepting configuration.
\begin{tabular}{c|c|c}
\toprule
Configuration & Needed suffix & $d_H$ \\
\midrule
$(q_f,Z_0)$ & (already accepting) & 0 \\
$(q_0,[Z_0])$ & \texttt{x} & 1 \\
$(q_0,[Z_0\texttt{(}])$ & \texttt{)}\texttt{x} & 2 \\
$(q_0,[Z_0\texttt{((}])$ & \texttt{))x} & 3 \\
$(q_0,[Z_0\texttt{(((}])$ & \texttt{)))x} & 4 \\
$(q_0,[Z_0\texttt{((((}])$ & \texttt{))))x} & 5 \\
\bottomrule
\end{tabular}

\subsection{Step‑by‑Step Execution}

We initialise the beam with the configuration $(q_0,Z_0,\epsilon)$.

\subsubsection{Step 1 ($T=3$)}

From state $(q_0,Z_0)$, the admissible tokens are \texttt{'('} (valid transition, $d_H=2$) and \texttt{'x'} (valid transition, leads to $q_f$, $d_H=0$), while \texttt{')'} has no valid transition given the current configuration and is therefore pruned. The model logits are $z(\texttt{(}) = -0.5$, $z(\texttt{x}) = -1.0$, then $z_{\max} = -0.5$.
After DGS and log-softmax computation, the scores become $s'(\texttt{(}) = -0.523$, $s'(\texttt{x}) = -0.898$. The beams after the first step are:
\begin{center}
\begin{tabular}{c|c|c|c}
\toprule
Sequence & Score & State & Stack \\
\midrule
\texttt{(} & $-0.523$ & $q_0$ & $Z_0\texttt{(}$ \\
\texttt{x} & $-0.898$ & $q_f$ (accepting) & $Z_0$ \\
\bottomrule
\end{tabular}
\end{center}

\subsubsection{Step 2 ($T=2$)}

The sequence “x” is already in an accepting configuration, and any addition of tokens would violate the constraints. For the sequence \texttt{'('}, the configuration is $(q_0,Z_0\texttt{(})$; the possible tokens are \texttt{')'} (valid transition, $d_H=1$) and \texttt{'('} (valid but $d_H=3$, which exceeds the remaining budget of 2, so it is pruned), while \texttt{'x'} has no valid transition given the current configuration and is therefore pruned. Only one token remains, hence the score does not change:
\[
s'(\texttt{()}) = -0.523 + 0 = -0.523.
\]
The beam at step 2 is:

\begin{center}
\begin{tabular}{c|c|c|c}
\toprule
Sequence & Score & State & Stack \\
\midrule
\texttt{x} & $-0.898$ & $q_f$ & $Z_0$ \\
\texttt{()} & $-0.523$ & $q_0$ & $Z_0$ \\
\bottomrule
\end{tabular}
\end{center}

\subsubsection{Step 3 ($T=1$)}

The sequence \texttt{'x'} remains in an accepting configuration. For \texttt{'()'}, the configuration is $(q_0,Z_0)$; admissible tokens are \texttt{'x'} (valid transition, $d_H=0$) and \texttt{'('} (valid but $d_H=2$, which exceeds the remaining budget of 1, so it is discarded). Again, only one token remains, so
\[
s'(\texttt{()x}) = -0.523.
\]
The final beam contains two accepting sequences:

\begin{center}
\begin{tabular}{c|c|c}
\toprule
Sequence & Score & Accepting? \\
\midrule
\texttt{x} & $-0.898$ & Yes \\
\texttt{()x} & $-0.523$ & Yes \\
\bottomrule
\end{tabular}
\end{center}

\subsection{Final Result}
Among the two sequences, we select the one with the highest score:
\[
\boxed{\texttt{()x}} \quad \text{because } -0.523 > -0.898.
\]
Thus, the output generated by our method for this example is the string \texttt{()x}.
This simple case demonstrates how the beam search, guided by the distance function $d_H$, efficiently explores the space of derivations, discards candidates that cannot lead to a valid string within the token budget, and finally picks the most plausible solution.

\begin{table*}[!t]
\centering

\begin{tabular}{c|c|c|c|c}
\toprule
Current State & Input & Stack Top & Next State & Stack Operation \\
\midrule
$q_0$ & \texttt{(} & $Z_0$ & $q_0$ & push \texttt{(} \\
$q_0$ & \texttt{(} & \texttt{(} & $q_0$ & push \texttt{(} \\
$q_0$ & \texttt{)} & \texttt{(} & $q_0$ & pop \texttt{(}\\
$q_0$ & \texttt{x} & $Z_0$ & $q_f$ & -- \\
\bottomrule
\end{tabular}

\caption{PDA transitions for the running example.}
\label{tab:pda}
\end{table*}

\end{document}